\documentclass{article}
\usepackage{ijcai19}

\usepackage{times}
\usepackage{soul}
\usepackage{url}
\usepackage[hidelinks]{hyperref}
\usepackage[utf8]{inputenc}
\usepackage[small]{caption}
\usepackage{graphicx}
\usepackage{amsmath}
\usepackage{booktabs}
\usepackage{algorithm}
\usepackage{algorithmic}
\usepackage{subcaption}
\usepackage{amssymb}
\usepackage{mathtools}
\newtheorem{assumption}{Assumption}
\newtheorem{theorem}{Theorem}
\newtheorem{lemma}{Lemma}
\newtheorem{corollary}{Corollary}
\newtheorem{definition}{Definition}
\newenvironment{proof}{{\noindent\it\bfseries Proof.\quad}}{\hfill $\square$\par}

\title{Multi-Objective Generalized Linear Bandits}

\author{
Shiyin Lu$^1$
\and
Guanghui Wang$^1$\and
Yao Hu$^{2}$\And
Lijun Zhang$^{1}$
\affiliations
$^1$National Key Laboratory for Novel Software Technology, Nanjing University, Nanjing 210023, China\\
$^2$YouKu Cognitive and Intelligent Lab, Alibaba Group, Beijing 100102, China
\emails
\{lusy, wanggh, zhanglj\}@lamda.nju.edu.cn,
yaoohu@alibaba-inc.com
}

\begin{document}

\maketitle

\begin{abstract}
In this paper, we study the multi-objective bandits (MOB) problem, where a learner repeatedly selects one arm to play and then receives a reward vector consisting of multiple objectives. MOB has found many real-world applications as varied as online recommendation and network routing. On the other hand, these applications typically contain contextual information that can guide the learning process which, however, is ignored by most of existing work. To utilize this information, we associate each arm with a context vector and assume the reward follows the generalized linear model (GLM). We adopt the notion of Pareto regret to evaluate the learner's performance and develop a novel algorithm for minimizing it. The essential idea is to apply a variant of the online Newton step to estimate model parameters, based on which we utilize the upper confidence bound (UCB) policy to construct an approximation of the Pareto front, and then uniformly at random choose one arm from the approximate Pareto front. Theoretical analysis shows that the proposed algorithm achieves an $\tilde O(d\sqrt{T})$ Pareto regret, where $T$ is the time horizon and $d$ is the dimension of contexts, which matches the optimal result for single objective contextual bandits problem. Numerical experiments demonstrate the effectiveness of our method.
\end{abstract}

\section{Introduction}
Online learning under bandit feedback is a powerful paradigm for modeling sequential decision-making process arising
in various applications such as medical trials, advertisement placement,
 and network routing \cite{bubeck2012regret}.
In the classic stochastic multi-armed bandits (MAB) problem, at each round a learner firstly selects one arm
to play and then obtains a reward drawn from a fixed but unknown probability distribution associated with the selected arm.
The learner's goal is to minimize the regret, which is defined as the difference between the cumulative reward of the learner and that of the best arm in
hindsight.
Algorithms designed for this problem need to strike a balance between exploration and exploitation, i.e., 
identifying the best arm by trying different arms 
while spending as much as possible on the seemingly optimal arm.

A natural extension of MAB is the multi-objective multi-armed bandits (MOMAB), proposed by \citeauthor{drugan2013designing} \shortcite{drugan2013designing}, where
the reward pertaining to an arm is a multi-dimensional vector instead of a scalar
value. 
In this setting, different arms are compared according to Pareto order between their reward vectors, and those arms whose rewards are not inferior to that of any other arms are called Pareto optimal arms, all of which constitute the Pareto front. The standard metric is the Pareto regret, which measures the cumulative gap between the reward of the learner and that of the Pareto front. The task here is to design online algorithms that minimize the Pareto regret by judiciously selecting seemingly Pareto optimal arms based on historical observation, while ensuring fairness, that is, treating each Pareto optimal arm as equally as possible.
MOMAB is motivated by real-world applications involved with multiple optimization
objectives, e.g., novelty and diversity in recommendation systems
\cite{rodriguez2012multiple}.
On the other hand, the aforementioned real-world applications typically contain auxiliary 
information (contexts) that can guide the decision-making process, such as user profiles in recommendation systems \cite{li2010contextual}, which is ignored by MOMAB.

To incorporate this information into the decision-making process,
 \citeauthor{turgay2018multi} \shortcite{turgay2018multi} extended MOMAB to the
multi-objective contextual bandits (MOCB). In MOCB, the learner is endowed with contexts before choosing arms and the reward he receives in each
round obeys a distribution whose expectation depends on the contexts and the chosen arm.
\citeauthor{turgay2018multi} \shortcite{turgay2018multi} assumed that the learner has a prior knowledge of the similarity information that relates distances between the
context-arm pairs to those between the expected rewards. Under this assumption, they proposed an algorithm called Pareto contextual zooming which is built upon 
the contextual zooming method \cite{slivkins2014contextual}. However, the Pareto regret of their algorithm is 
$\tilde O(T^{1 - 1 / (2+d_p)})$, where $d_p$ is the Pareto zooming dimension, which is almost linear in $T$ when $d_p$ is large (say, $d_p = 10$)
and hence hinders the application of their algorithm to broad domains.

To address this limitation, we formulate the multi-objective contextual bandits under a different assumption---the parameterized realizability assumption, which has been extensively studied in single objective contextual bandits 
\cite{auer2002using,Dani2008StochasticLO}.
Concretely, we model the
context associated with an arm as a $d$-dimensional vector $x \in \mathbb{R}^d$ and for the sake of clarity denote the arm
by $x$. The reward vector $y$ pertaining to an arm $x$ consists of $m$ objectives. The value of each objective is drawn according to 
the generalized linear model \cite{Neld1972GLM} such that
\begin{equation*}
    \mathbb{E}[y^i \vert x] = \mu_i (\theta_i^\top x),\; i=1,\ldots,m 
\end{equation*}
where $y^i$ represents the $i$-th component of $y$, $\theta_1, \ldots, \theta_m$ are vectors of unknown coefficients, 
and $\mu_1, \ldots, \mu_m$ are link functions. We refer to this formulation as multi-objective generalized linear bandits (MOGLB), which
is very general and covers a wide range of problems, such as stochastic linear bandits
\cite{auer2002using,Dani2008StochasticLO} and online stochastic linear optimization under binary feedback \cite{zhang2016online}, where
the link functions are the identity function and the logistic function respectively.

To the best of our knowledge, this is the first work that investigates the generalized linear bandits (GLB) in multi-objective scenarios.
Note that a naive application of existing GLB algorithms to a specific objective does not work, because it could favor those Pareto optimal arms that achieve maximal reward in this objective, which harms the fairness.
To resolve this problem, we develop a novel
algorithm named MOGLB-UCB.
Specifically, we employ a variant of the online Newton step to estimate unknown coefficients
and utilize the upper confidence bound policy to construct an approximate Pareto front, from which the arm is then
pulled uniformly at random.
Theoretical analysis shows
that the proposed algorithm enjoys a Pareto regret bound of $\tilde O(d\sqrt{T})$, where $T$ is
the time horizon and $d$ is the dimension of contexts. This bound is sublinear in $T$ regardless of the dimension and 
matches the optimal regret bound for single objective contextual bandits problem. Empirical results demonstrate that our algorithm can not only minimize Pareto regret but also ensure  fairness.

\section{Related Work}
In this section, we review the related work on stochastic contextual bandits, parameterized contextual bandits, and multi-objective bandits.

\subsection{Stochastic Contextual Bandits}
In the literature, there exist many formulations of the stochastic contextual bandits, under different
assumptions on the problem structure, i.e., the mechanism of context arrivals
and the relation between contexts and rewards.

One category of assumption says that the context and reward follow a fixed but unknown joint distribution.
This problem is first considered by \citeauthor{langford2008epoch} \shortcite{langford2008epoch}, who proposed an efficient algorithm named Epoch-Greedy.
However, the regret of their algorithm is $O(T^{2/3})$, which is suboptimal when compared to inefficient algorithms such as
Exp4 \cite{auer2002nonstochastic}. Later on, efficient and optimal algorithms that attain an $\tilde O(T^{1/2})$ regret are developed 
\cite{Dudik11,agarwal2014taming}.

Another line of work \cite{kleinberg2008multi,bubeck2009online,lu2010contextual,slivkins2014contextual,pmlr-v97-lu19c} assume that the relation between the rewards and the contexts can be modeled by a Lipschitz function. \citeauthor{lu2010contextual}
 \shortcite{lu2010contextual} established an $\Omega(T^{1-1/(2+d_p)})$ lower
bound on regret under this setting and proposed the Query-Ad-Clustering algorithm, which attains an $\tilde O(T^{1-1/(2+d_c)})$
regret, where $d_p$ and $d_c$ are the packing dimension and the covering dimension of the similarity space respectively.
 \citeauthor{slivkins2014contextual} \shortcite{slivkins2014contextual} developed the contextual zooming algorithm,
which enjoys a regret bound of $\tilde O(T^{1-1/(2+d_z)})$, where $d_z$ is the zooming dimension
of the similarity space. 

\subsection{Parameterized Contextual Bandits}
In this paper, we focus on the parameterized contextual bandits, where
each arm is associated with a $d$-dimensional context vector and the expected reward is modeled as a parameterized function of the arm's context. \citeauthor{auer2002using} \shortcite{auer2002using} considered the linear case of this problem under the name of stochastic linear bandits (SLB) and
 developed a complicated algorithm called SuperLinRel, which attains an
$\tilde O(({\log{K}})^{3/2}\sqrt{dT})$ regret, assuming the arm set is finite. Later, 
\citeauthor{Dani2008StochasticLO} \shortcite{Dani2008StochasticLO} proposed a
much simpler algorithm named ConfidenceBall$_2$, which enjoys a regret bound of $\tilde O(d\sqrt{T})$ and can be used
for infinite arm set.

\citeauthor{filippi2010parametric} \shortcite{filippi2010parametric} extended SLB to the generalized linear bandits, where 
the expected reward is a composite function of the arm's context. The inside function is linear and the outside
function is certain link function. The authors proposed a UCB-type algorithm that enjoys a regret bound of $\tilde O(d\sqrt{T})$.
However, their algorithm is not efficient since it needs to store the whole learning history and perform batch computation to estimate
the function. \citeauthor{zhang2016online} \shortcite{zhang2016online} studied a particular case of GLB in which the reward is 
generated by the logit model and  developed an efficient algorithm named OL$^2$M, which attains an $\tilde O(d\sqrt{T})$ regret. Later,
\citeauthor{jun2017scalable} \shortcite{jun2017scalable} extended OL$^2$M to generic GLB problems.

\subsection{Multi-Objective Bandits}
The seminal work of \citeauthor{drugan2013designing} \shortcite{drugan2013designing} proposed the standard formulation of the multi-objective
multi-armed bandits and introduced the notion of Pareto regret as the performance
measure. By making use of the UCB technique, they developed an algorithm that enjoys a Pareto regret bound of $O(\log{T})$. 
\citeauthor{turgay2018multi} \shortcite{turgay2018multi} extended MOMAB to the contextual setting with the similarity information assumption. 
Based on the contextual zooming method \cite{slivkins2014contextual}, the authors proposed an algorithm called Pareto contextual zooming whose
Pareto regret is $\tilde O(T^{1 - 1 / (2+d_p)})$, where $d_p$ is the Pareto zooming dimension.
Another related line of the MOMAB researches \cite{drugan2014scalarization,auer2016pareto} 
study the best arm identification problem. The focus of those papers is to identify all Pareto optimal arms
within a fixed budget.

\section{Multi-Objective Generalized Linear Bandits}
We first introduce notations used in this paper, next describe the learning model, then present our algorithm, and finally state its theoretical guarantees.

\subsection{Notation}
Throughout the paper, we use the subscript to distinguish different objects (e.g., scalars, vectors, functions) and 
superscript to identify the component of an object. For example,
$y_t^i$ represents the $i$-th component of the vector $y_t$. For the sake of clarity, we denote the $\ell_2$-norm by $\Vert \cdot \Vert$.
The induced matrix norm associated with a positive definite matrix $A$ is defined as $\Vert x \Vert_{A} \coloneqq \sqrt{x^\top A x} $. 
We use $\mathcal{B}_R \coloneqq \{ x ~\vert~ \Vert x \Vert \leq R \}$ to denote a centered ball whose radius is $R$. 
Given a positive semidefinite matrix $P$, the generalized projection of a point $x$ onto a convex set $\mathcal{W}$ is defined as
$    \Pi_{\mathcal{W}}^P [x] \coloneqq \mathop{\arg\min}_{y \in \mathcal{W}} (y - x)^\top P (y-x) $.
Finally, $[n] \coloneqq \{ 1, 2, \ldots, n \}$.

\subsection{Learning Model}
We now give a formal description of the learning model investigated in this paper.
\subsubsection{Problem Formulation}
We consider the multi-objective bandits problem under the GLM realizability assumption.
Let $m$ denote the number of objectives and $\mathcal{X} \subset
 \mathbb{R}^d$ be the arm set. 
In each round $t$, a learner selects an arm $x_t \in \mathcal{X}$ to play and then receives a stochastic reward vector $y_t \in \mathbb{R}^m$ consisting of $m$ objectives.
We assume each objective $y_t^i$ is generated according to the GLM such that for $i = 1,2,\ldots,m$,
\begin{equation*}
    \Pr(y_t^i \vert x_t) = h_i(y_t^i, \tau_i) \exp\left(\frac{y_t^i \theta_i^\top x_t - g_i(\theta_i^\top x_t)}{\tau_i}\right)
\end{equation*}
where $\tau_i$ is the
dispersion parameter, $h_i$ is a normalization function, $g_i$ is a convex function, and $\theta_i$ is 
a vector of unknown coefficients.
Let $\mu_i = g_i'$ denote the so-called link function, which is monotonically increasing due to the convexity of $g_i$.
It is easy to show $\mathbb{E}[y_t^i \vert x_t] = \mu_i(\theta_i^\top x_t)$.

A remarkable member of the GLM family is the logit model in which the reward is one-bit, i.e., $y \in \{0, 1\}$ \cite{zhang2016online}, and satisfies
\begin{equation*}
    \Pr(y=1 \vert x) = \frac{1}{1+\exp{(-\theta^\top x)}} .
\end{equation*}
Another well-known binary model belonging to the GLM is the probit model, which takes the following form
\begin{equation*}
    \Pr(y=1 \vert x) = \Phi(\theta^\top x)
\end{equation*}
where $\Phi(\cdot)$ is the cumulative distribution function of the standard normal distribution.

Following previous studies \cite{filippi2010parametric,jun2017scalable}, we make standard assumptions as follows.
\begin{assumption}
        The coefficients $\theta_1, \ldots, \theta_m$ are bounded by $D$, i.e., $\Vert \theta_i \Vert \leq D, \forall i \in [m]$.
\end{assumption}
\begin{assumption}
        The radius of the arm set $\mathcal{X}$ is bounded by $1$, i.e., $\Vert x \Vert \leq 1, \forall x \in \mathcal{X}$.
\end{assumption}
\begin{assumption}
        For each $i \in [m]$, the link function $\mu_i$ is $L$-Lipschitz on $[-D, D]$ and continuously differentiable on $(-D, D)$. Furthermore,
        we assume that $\mu_i'(z) \geq \kappa > 0, z \in (-D, D) $ and $\vert \mu_i(z) \vert \leq U, z \in [-D, D]$.
\end{assumption}
\begin{assumption}
        There exists a positive constant $R$ such that $\vert y_t^i \vert \leq R, \forall t \in [T], i \in [m]$ holds almost surely.
\end{assumption}
\subsubsection{Performance Metric}
According to the properties of the GLM, for any arm $x \in \mathcal{X}$ that is played, its expected reward is 
a vector of $[\mu_1(\theta_1^\top x), \mu_2(\theta_2^\top x), \ldots, \mu_m(\theta_m^\top x)] \in \mathbb{R}^m$.
With a slight abuse of notation, we denote it by $\mu_x$. We compare different arms by their expected rewards and adopt the notion of
Pareto order.
\begin{definition}[Pareto order]
    Let $u, v \in \mathbb{R}^m$ be two vectors.
    \begin{itemize}
        \item $u$ dominates $v$, denoted by $v \prec u$, if and only if $\;\forall i \in [m], v^i \leq u^i$ and $\exists j \in [m], u^j > v^j$.
        \item $v$ is not dominated by $u$, denoted by $v \nprec u$, if and only if $v = u$ or $\;\exists i \in [m], v^i > u^i$. 
        \item $u$ and $v$ are incomparable, denoted by $u \Vert v$, if and only if either vector is not dominated by the other,
        i.e., $u \nprec v$ and $v \nprec u$.
    \end{itemize}
\end{definition}

Equipped with the Pareto order, we can now define the Pareto optimal arm.
\begin{definition}[Pareto optimality]
    Let $x \in \mathcal{X}$ be an arm.
    \begin{itemize}
        \item $x$ is Pareto optimal if and only if its expected reward is not dominated by that of any arm in $\mathcal{X}$,
        i.e., $\;\forall x' \in \mathcal{X}, \mu_x \nprec \mu_x'$.
        \item The set comprised of all Pareto optimal arms is called Pareto front, denoted by $\mathcal{O}^*$.
    \end{itemize}
\end{definition}
It is clear that all arms in the Pareto front are incomparable. In single objective bandits problem, the standard metric to measure the
 learner's performance is regret defined as the 
difference between the cumulative reward of the learner and that of the optimal arm in hindsight. In order to extend such metric to
multi-objective setting, we introduce the notion of Pareto suboptimality gap \cite{drugan2013designing} to measure the difference between the learner's reward
and that of the Pareto optimal arms.
\begin{definition}[Pareto suboptimality gap, PSG]
    Let $x$ be an arm in $\mathcal{X}$. Its Pareto suboptimality gap $\Delta x$ is defined as the minimal scalar $\epsilon \geq 0$ 
    such that $x$ becomes Pareto optimal after adding $\epsilon$ to all entries of its expected reward. Formally, 
    \begin{equation*}
        \Delta x \coloneqq \inf \; \{\epsilon \;\vert\; (\mu_x+\epsilon) \nprec \mu_{x'}, \forall x' \in \mathcal{X} \} .
    \end{equation*}
\end{definition}
We evaluate the learner's performance using the (pseudo) Pareto regret \cite{drugan2013designing} defined as the cumulative Pareto suboptimality gap 
of the arms pulled by the learner.
\begin{definition}[Pareto regret, PR]
    Let $x_1, x_2, \ldots, x_T$ be the arms pulled by the learner. The Pareto regret is defined as
    \begin{equation*}
        PR(T) \coloneqq \sum_{t=1}^T \Delta x_t .
    \end{equation*}
\end{definition}

\subsection{Algorithm}
The proposed algorithm, termed MOGLB-UCB, is outlined in Algorithm \ref{alg:001}.
Had we known all coefficients $\theta_1, \ldots, \theta_m$ in advance, we could compute the Pareto front 
directly and always pull the Pareto optimal arms, whose
Pareto suboptimality gaps are zero. Motivated by this observation,
we maintain an arm set $\mathcal{O}_t$ as an approximation to the Pareto front $\mathcal{O}^*$ and always play arms in $\mathcal{O}_t$. To encourage fairness, we draw an arm $x_t$ from $\mathcal{O}_t$ uniformly at random to play (Step 3).
The approximate Pareto front is initialized to be $\mathcal{X}$ and is updated as follows.

In each round $t$, after observing the reward vector $y_t$, we make an estimation denoted
 by $\hat \theta_{t+1, i}$ for each coefficients $\theta_i$ (Steps 4-7). 
 Let $\mathcal{H}_t \coloneqq \{ (x_1, y_1), (x_2, y_2), \ldots, (x_t, y_t) \}$
 be the learning history up to round $t$. A natural approach is to use the maximum log-likelihood estimation:
\begin{align*}
    \hat \theta_{t+1, i} & = \mathop{\arg \max}_{\Vert \theta \Vert \leq D} \sum_{s=1}^t \log {\Pr(y_s^i \vert x_s)} \\
                       & = \mathop{\arg \min}_{\Vert \theta \Vert \leq D} \sum_{s=1}^t - y_s^i \theta^\top x_s + g_i(\theta^\top x_s) .
\end{align*}
Despite its simplicity, this approach is inefficient since it needs to store the whole learning history and perform batch computation
in each round, which makes its space and time complexity grow at least linearly with $t$.

To address this drawback, we utilize an online learning method to estimate the unknown coefficients and 
 construct confidence sets.
Specifically, for each objective $i \in [m]$, let $\ell_{t, i}$ denote the surrogate loss function in round $t$,
 defined as
\begin{equation*}
     \ell_{t,i} (\theta) \coloneqq -y_t^i \theta^\top x_t + g_i(\theta^\top x_t) .
\end{equation*}
We employ a variant of the online Newton step and compute $\hat \theta_{t+1, i}$ by
\begin{equation}
    \label{eq:ons}
    \begin{split}
        \hat \theta_{t+1, i} & = \mathop{\arg \min}_{\Vert \theta \Vert \leq D} \frac{\Vert \theta - \hat \theta_{t, i} \Vert_{Z_{t+1}}^2}{2}
        + \theta^\top \nabla \ell_{t, i}(\hat \theta_{t, i})  \\
        & = \Pi_{\mathcal{B}_D}^{Z_{t+1}} [\hat \theta_{t,i} - Z_{t+1}^{-1} \nabla \ell_{t, i}(\hat \theta_{t, i})]
    \end{split}
\end{equation}
where 
\begin{equation}
    \label{eq:zt}
    Z_{t+1} = Z_{t} + \frac{\kappa}{2}{x_t x_t^\top} = \lambda I_d + \frac{\kappa}{2} \sum_{s=1}^t x_s x_s^\top
\end{equation} 
and $\nabla \ell_{t, i}(\hat \theta_{t, i}) = -y_t^i x_t + 
\mu_i(\hat \theta_{t,i}^\top x_t) x_t$.
Based on this estimation, we construct a confidence set $\mathcal{C}_{t+1, i}$ (Step 8) such that the true vector of coefficients $\theta_i$ falls into 
it with high probability. According to the theoretical analysis (Theorem 1), we define $\mathcal{C}_{t+1, i}$ as an
ellipsoid centered at $\hat \theta_{t+1, i}$:
\begin{equation}
    \label{eq:con}
    \mathcal{C}_{t+1, i} \coloneqq \{ \theta : \Vert \theta - \hat \theta_{t+1, i} \Vert_{Z_{t+1}}^2 \leq \gamma_{t+1} \}
\end{equation}
where $\gamma_{t+1}$ is defined in (\ref{def:gamma}).

Then, we adopt the principle of ``optimism in face of uncertainty'' to balance exploration and exploitation. Specifically,
for each arm $x \in \mathcal{X}$, we compute the upper confidence bound of its expected reward $\hat \mu_{t+1, x}$ (Step 11) as
\begin{equation}
    \label{eq:upp}
    \hat \mu_{t+1, x}^i = \max_{\theta \in \mathcal{C}_{t+1, i}} \mu_i(\theta^\top x), \; i=1,2,\ldots,m .
\end{equation} 
Based on it, we define the empirical Pareto optimality.
\begin{definition}[Empirical Pareto optimality]
        An arm $x \in \mathcal{X}$ is empirically Pareto optimal if and only if the upper confidence bound of its expected reward
        is not dominated by that of any arm in $\mathcal{X}$, i.e.,
        $\forall x' \in \mathcal{X}, \hat \mu_{t+1, x} \nprec \hat \mu_{t+1, x'}$.
\end{definition}
Finally, we update the approximate Pareto front $\mathcal{O}_t$ (Step 13) by finding all empirically Pareto optimal arms:
\begin{equation*}
    \mathcal{O}_{t+1} = \{ x \in \mathcal{X} ~\vert~ \forall x' \in \mathcal{X}, \hat \mu_{t+1, x} \nprec \hat \mu_{t+1, x'}  \} .
\end{equation*}

\begin{algorithm}[tbp]
    \caption{MOGLB-UCB} \label{alg:001}
    \begin{algorithmic}[1]
    \REQUIRE Regularization parameter $\lambda \geq \max (1, \kappa / 2)$
    \STATE Initialize $Z_1 = \lambda I_d, \hat \theta_{1, 1} = \cdots = \hat \theta_{1,m} = \text{\bf{0}}, \mathcal{O}_1 = \mathcal{X}$
    \FOR{$t=1,2,\ldots,T$}
        \STATE Pull an arm $x_t$ from the approximate Pareto front $\mathcal{O}_t$ uniformly at random
        \STATE Observe the reward vector $y_t$
        \STATE Update $Z_{t+1} = Z_{t} + \frac{\kappa}{2}{x_t x_t^\top}$
        \FOR{$i=1,2,\ldots,m$}
            \STATE Compute the estimation $\hat \theta_{t+1, i}$ by formula (\ref{eq:ons})
            \STATE Construct the confidence set $\mathcal{C}_{t+1, i}$ by formula (\ref{eq:con})
        \ENDFOR
        \FOR{each $x \in \mathcal{X}$}
            \STATE Compute the upper confidence bound $\hat \mu_{t+1, x}$ by formula (\ref{eq:real-upp})
        \ENDFOR
        \STATE Update the approximate Pareto front $\mathcal{O}_{t+1} = \{ x \in \mathcal{X} \; \vert \; \forall x' \in \mathcal{X},
         \hat \mu_{t+1, x} \nprec \hat \mu_{t+1, x'}  \}$
    \ENDFOR
    \end{algorithmic}
\end{algorithm}

Note that the computation in (\ref{eq:upp}) involves the link function $\mu_i$, which may be very complicated. Fortunately, 
thanks to the fact that the updating mechanism only relies on the Pareto order between arms' rewards and 
the link function is monotonically increasing, we can replace
(\ref{eq:upp}) by
\begin{equation}
    \label{eq:prac-upp}
    \hat \mu_{t+1, x}^i = \max_{\theta \in \mathcal{C}_{t+1, i}} \theta^\top x, \; i=1,2,\ldots,m .
\end{equation}
Furthermore, by standard algebraic manipulations, the above optimization
problem can be rewritten in a closed form:
\begin{equation}
    \label{eq:real-upp}
    \hat \mu_{t+1, x}^i = \hat \theta_{t+1, i}^\top x + \sqrt{\gamma_{t+1}} \Vert x \Vert_{Z_{t+1}^{-1}} .
\end{equation}

\subsection{Theoretical Guarantees}
We first show that the confidence sets constructed in each round
contain the true coefficients with high probability.
\begin{theorem}
    \label{thm:001}
    With probability at least $1 - \delta$,
    \begin{equation}
        \label{eq:thm1}
        \Vert \theta_i - \hat \theta_{t+1, i} \Vert_{Z_{t+1}}^2 \leq \gamma_{t+1}, \; \forall i \in [m], \forall t \geq 0
    \end{equation}
    where
    \begin{equation}
        \label{def:gamma}
        \begin{split}
            \gamma_{t+1} = & \frac{16(R+U)^2}{\kappa} \log{\left(\frac{m}{\delta} 
            \sqrt{1 + 4D^2t} \right)} + \lambda D^2 \\
            & + \frac{2(R+U)^2}{\kappa} \log{\frac{\det{(Z_{t+1})}}{\det{(Z_1)}}} + \frac{\kappa}{2} .
        \end{split}
    \end{equation}
\end{theorem}
The main idea lies in exploring the properties of the surrogate loss function (Lemmas \ref{lm:001} and \ref{lm:002}), analyzing the estimation method (Lemma \ref{lm:003}), and utilizing the self-normalized bound for martingales (Lemma \ref{lm:whp}).

We then investigate the data-dependent item $\log{\frac{\det{(Z_{t+1})}}{\det{(Z_1)}}}$ appearing in
the definition of $\gamma_{t+1}$ and bound the width of the confidence set by the following corollary, which is a direct consequence of
Lemma 10 in \citeauthor{abbasi2011improved} \shortcite{abbasi2011improved}.
\begin{corollary}
    \label{cor:1}
    For any $t \geq 0$, we have
    \begin{equation*}
        \log{\frac{\det{(Z_{t+1})}}{\det{(Z_1)}}} \leq d \log{ \left( 1 + \frac{\kappa t}{2 \lambda d} \right) }
    \end{equation*}
    and hence
    \begin{equation*}
        \label{eq:upp-gamma}
        \gamma_{t+1} \leq O(d \log{t}) .
    \end{equation*}
\end{corollary}

Finally, we present the Pareto regret bound of our algorithm, which is built upon on Theorem \ref{thm:001}.
\begin{theorem}
    \label{thm:002}
    With probability at least $1 - \delta$,
    \begin{equation*}
        \begin{split}
            PR(T)  \leq 4L \sqrt{\frac{dT}{\kappa} \log{\left( 1 + \frac{\kappa T}{2\lambda d} \right)} \gamma_{T+1} }
        \end{split}
    \end{equation*}
    where $\gamma_{T+1}$ is defined in (\ref{def:gamma}).
\end{theorem}
\paragraph{Remark.} The above theorem implies that our algorithm enjoys a Pareto regret bound of $\tilde O (d\sqrt{T})$, which 
matches the optimal result for single objective GLB problem. Futhermore, in contrast to the $\tilde O(T^{1 - 1 / (2+d_p)})$
 Pareto regret bound of \citeauthor{turgay2018multi} \shortcite{turgay2018multi},
 which is almost linear in $T$ when the Pareto zooming dimension $d_p$ is large,
our bound grows sublinearly with $T$ regardless of the dimension.

\section{Analysis}
In this section, we provide proofs of the theoretical results.
\subsection{Proof of Theorem \ref{thm:001}}
We follow standard techniques for analyzing the confidence sets \cite{zhang2016online,jun2017scalable}
and start with the following lemma.
\begin{lemma}
    \label{lm:001}
    For any $t \geq 1$ and $i \in[m]$, the inequality
    \begin{equation*}
        \ell_{t, i} (u) - \ell_{t, i} (v) \geq \nabla \ell_{t, i}(v)^\top (u - v) +  \frac{\kappa}{2}
         (u^\top x_t - v^\top x_t)^2
    \end{equation*}
    holds for any $u, v \in \mathcal{B}_{D}$.
\end{lemma}

\begin{proof}
    Define $\tilde \ell_{t, i}(z) = -y_t^i z + g_i(z), z \in [-D, D]$. By Assumption 3 we have
    \begin{equation*}
        \tilde \ell''_{t, i}(z) = \mu_i'(z) \geq \kappa > 0, \; \forall z \in (-D, D)
    \end{equation*}
    which implies $\tilde \ell_{t, i}(z)$ is $\kappa$-strongly convex on $[-D, D]$. Thus, we have $\forall z, z' \in [-D, D]$,
    \begin{equation}
        \label{eq:lm:002}
            \tilde \ell_{t, i} (z) - \tilde \ell_{t, i} (z') \geq \tilde \ell'_{t, i}(z') (z - z') +
            \frac{\kappa}{2}(z - z')^2 .
    \end{equation} 
    Note that
    \begin{equation*}
        \begin{split}
            \ell_{t, i} (u) = \tilde  \ell_{t, i} (u^\top x_t), \quad \ell_{t, i} (v) = \tilde  \ell_{t, i} (v^\top x_t), \\
            \nabla \ell_{t, i}(v)^\top (u - v) = \tilde  \ell'_{t, i} (v^\top x_t) (u-v)^\top x_t .
        \end{split}
    \end{equation*}
    We complete the proof by substituting $z = u^\top x_t$ and $z' = v^\top x_t$ into (\ref{eq:lm:002}).
\end{proof}

Let $f_{t, i}(\theta) = \mathbb{E}_{y_t^i} [\ell_{t, i}(\theta)]$ be the conditional expectation over $y_t^i$.
The following lemma shows that $\theta_i$ is the minimum point of $f_{t, i}(\theta)$ on the bounded domain.
\begin{lemma}
    \label{lm:002}
    For any $t \geq 1$ and $i \in[m]$, we have
    $$
        f_{t, i}(\theta) - f_{t, i}(\theta_i) \geq 0, \; \forall \theta \in \mathcal{B}_{D} .
    $$
\end{lemma}
\begin{proof}
    Fix $t$ and $i$. For any $\theta \in \mathcal{B}_{D}$, we have
    \begin{align*}
        & \, f_{t, i}(\theta) - f_{t, i}(\theta_i) \\
         = & \, \mathbb{E}_{y_t^i} [\ell_{t, i}(\theta) - \ell_{t, i}(\theta_i)] \\
         = & \, g_i(\theta^\top x_t) - g_i(\theta_i ^\top x_t) + \mu_i(\theta_i ^\top x_t) 
         (\theta_i ^\top x_t - \theta ^\top x_t) \\
         \geq & \, g_i'(\theta_i ^\top x_t) (\theta ^\top x_t - \theta_i ^\top x_t) + \mu_i(\theta_i ^\top x_t) 
         (\theta_i ^\top x_t - \theta ^\top x_t) \\
        = & \, 0 .
    \end{align*}
    where the second and the last equalities are due to the properties of GLM, and the inequality holds since $g_i$ is a convex function.
\end{proof}

To exploit the property of the estimation method in (\ref{eq:ons}), we introduce the following lemma from 
\citeauthor{zhang2016online} \shortcite{zhang2016online}.
\begin{lemma}
    \label{lm:003}
    For any $t \geq 1$ and $i \in[m]$,
    \begin{equation}
        \label{eq:lm:003}
        \begin{split}
            & \nabla \ell_{t, i}(\hat \theta_{t, i})^\top (\hat \theta_{t, i} - \theta_i) - \frac{1}{2} 
            {\Vert \nabla \ell_{t, i} (\hat \theta_{t, i}) \Vert}_{Z_{t+1}^{-1}}^2 \\
            \leq \, & \frac{1}{2}\left({\Vert \hat \theta_{t, i} - \theta_i \Vert}_{Z_{t+1}}^2 - \; {\Vert \hat \theta_{t+1, i} - \theta_i \Vert}_{Z_{t+1}}^2\right).
        \end{split}
    \end{equation}
\end{lemma}

To proceed, we bound the norm of $\nabla \ell_{t, i} (\hat \theta_{t, i})$ as follows.
\begin{equation}
    \label{eq:upp-norm}
    \begin{split}
        {\Vert \nabla \ell_{t, i} (\hat \theta_{t, i}) \Vert}_{Z_{t+1}^{-1}}^2 & = {\Vert -y_t^i x_t + 
        \mu_i(\hat \theta_{t,i}^\top x_t) x_t \Vert}_{Z_{t+1}^{-1}}^2 \\
        & \leq (R + U)^2 {\Vert x_t \Vert}_{Z_{t+1}^{-1}}^2
    \end{split}
\end{equation}
where the inequality holds because of Assumptions 3 and 4. We are now ready to prove Theorem \ref{thm:001}. By Lemma \ref{lm:001}, we have
\begin{equation*}
        \ell_{t, i} (\hat \theta_{t, i}) - \ell_{t, i} (\theta_i) \leq \nabla \ell_{t, i}(\hat \theta_{t, i})^\top (\hat \theta_{t, i} - \theta_i) -  \frac{\kappa}{2}(\theta_i^\top x_t - \hat \theta_{t, i}^\top x_t )^2 .
\end{equation*}
Taking expectation in both side and using Lemma 2, we obtain
\begin{equation*}
    \begin{split}
        0 \leq & \; f_{t, i} (\hat \theta_{t, i}) - f_{t, i} (\theta_i) \\
          \leq & \; \nabla f_{t, i}(\hat \theta_{t, i})^\top (\hat \theta_{t, i} - \theta_i)
          -  \frac{\kappa}{2}(\theta_i^\top x_t - \hat \theta_{t, i}^\top x_t )^2 \\
          = & \; \underbrace{ (\nabla f_{t, i}(\hat \theta_{t, i}) - \nabla \ell_{t, i}(\hat \theta_{t, i}))^\top (\hat \theta_{t, i} - \theta_i)}_{a_{t,i}} \\
           & \; - \frac{\kappa}{2}\underbrace{(\theta_i^\top x_t - \hat \theta_{t, i}^\top x_t )^2}_{b_{t, i}}
          + \nabla \ell_{t, i}(\hat \theta_{t, i})^\top (\hat \theta_{t, i} - \theta_i) \\
          \overset{(\ref{eq:lm:003}, \ref{eq:upp-norm})}{\leq}& \;  a_{t, i} - \frac{\kappa}{2} b_{t, i} + \frac{(R + U)^2}{2} {\Vert x_t \Vert}_{Z_{t+1}^{-1}}^2 \\
          & \; + \frac{1}{2}\left({\Vert \hat \theta_{t, i} - \theta_i \Vert}_{Z_{t+1}}^2 - \; {\Vert \hat \theta_{t+1, i} - \theta_i \Vert}_{Z_{t+1}}^2\right) \\
          = & \;  a_{t, i} - \frac{\kappa}{4} b_{t, i} + \frac{(R + U)^2}{2} {\Vert x_t \Vert}_{Z_{t+1}^{-1}}^2 \\
          & \; + \frac{1}{2}\left({\Vert \hat \theta_{t, i} - \theta_i \Vert}_{Z_{t}}^2 - \; {\Vert \hat \theta_{t+1, i} - \theta_i \Vert}_{Z_{t+1}}^2\right)
    \end{split}
\end{equation*}
where the last equality is due to $Z_{t+1} = Z_t + \frac{\kappa}{2} x_t x_t^\top$. 
Summing the above inequality over $1$ to $t$ and rearranging, we have
\begin{equation}
    \label{eq:thm1:main}
    \begin{split}
        {\Vert \hat \theta_{t+1, i} - \theta_i \Vert}_{Z_{t+1}}^2 \leq \; & \lambda D^2 + 2 \sum_{s=1}^t a_{s, i} - \frac{\kappa}{2} \sum_{s=1}^t b_{s, i} \\
        & + (R+U)^2 \sum_{s=1}^t {\Vert x_s \Vert}_{Z_{s+1}^{-1}}^2 .
    \end{split}
\end{equation}
We propose the following lemma to bound $\sum_{s=1}^t a_{s,i}$.
\begin{lemma}
    \label{lm:whp}
    With probability at least $1 - \delta$, for any $i \in [m]$ and $t \geq 1$,
    \begin{equation}
        \label{eq:thm1:whp}
        \begin{split}
            \sum_{s=1}^t a_{s, i} \leq \; & \frac{\kappa}{4} \sum_{s=1}^t b_{s,i} + \frac{\kappa}{4} \\
            & + \frac{8(R+U)^2}{\kappa} \log{\left(\frac{m}{\delta} 
            \sqrt{1 + 4D^2t} \right)} .
        \end{split}
    \end{equation}
\end{lemma}
\begin{proof}
Let $\eta_{s, i} = y_s^i - \mu_i(\theta_{i}^\top x_s)$, which is a martingale difference sequence. 
By Assumptions 3 and 4, $\vert \eta_{s, i} \vert \leq (R+U)$ holds almost surely, which implies
 that $\eta_{s, i}$ is $(R+U)$-sub-Gaussian. Thus, we can apply the self-normalized bound for martingales (\citeauthor{abbasi2012online}
 \citeyear{abbasi2012online}, Corollary 8) and use the union bound
  to obtain that
 with probability at least $1 - \delta$,
 \begin{equation*}
    \begin{split}
        \sum_{s=1}^t a_{s, i} \leq \, & (R+U) \cdot \sqrt{\left( 2+2 \sum_{s=1}^t b_{s,i} \right)} \\
        & \cdot \sqrt{\log{\left(\frac{m}{\delta} \sqrt{1 + \sum_{s=1}^t b_{s,i}} \right)}},\; \forall i \in [m], \forall t \geq 1 .
    \end{split}
\end{equation*}
We conclude the proof by noticing that $b_{s,i} \leq 4D^2$ and using the well-known inequality
$\sqrt{pq} \leq \frac{p}{c} + cq$, where we pick $c = \frac{8(R+U)}{\kappa}$.
\end{proof}

It remains to bound $\sum_{s=1}^t {\Vert x_s \Vert}_{Z_{s+1}^{-1}}^2$. To this end, we employ Lemma 12 in
 \citeauthor{hazan2007logarithmic} \shortcite{hazan2007logarithmic} and
obtain
\begin{equation}
    \label{eq:thm1:hazan}
    \begin{split}
        \sum_{s=1}^t {\Vert x_s \Vert}_{Z_{s+1}^{-1}}^2 & \leq \frac{2}{\kappa} \sum_{s=1}^t \log{\frac{\det{(Z_{s+1})}}{\det{(Z_s)}}} \\
                                                        & \leq \frac{2}{\kappa} \log{\frac{\det{(Z_{t+1})}}{\det{(Z_1)}}} .
    \end{split}
\end{equation}
We complete the proof by combining (\ref{eq:thm1:main}), (\ref{eq:thm1:whp}), and (\ref{eq:thm1:hazan}). \hfill $\square$\par
\subsection{Proof of Theorem \ref{thm:002}.}
By Theorem \ref{thm:001},
\begin{equation}
    \label{eq:thm2:whp}
    \theta_i \in \mathcal{C}_{t, i}, \forall i \in [m], \forall t \geq 1
\end{equation}
holds with probability at least $1 - \delta$. For each objective $i \in [m]$ and each round $t \geq 1$, we define
\begin{equation}
    \label{eq:thm2:tilde}
    \tilde \theta_{t, i} \coloneqq \mathop{\arg \max}_{\theta \in \mathcal{C}_{t, i}} \theta^\top x_t .
\end{equation}
Recall that $x_t$ is selected from $\mathcal{O}_t$, which implies that for any $x \in \mathcal{X}$, there exists an objective $j \in [m]$ such that
\begin{equation}
    \label{eq:th2:002}
    \hat \mu_{t, x_t}^j \geq \hat \mu_{t, x}^j .
\end{equation}
By definitions in (\ref{eq:prac-upp}) and (\ref{eq:thm2:tilde}), we have
\begin{equation}
    \label{eq:th2:003}
    \hat \mu_{t, x_t}^j = \tilde \theta_{t, j}^\top x_t, \quad \hat \mu_{t, x}^j = \max_{\theta \in \mathcal{C}_{t, j}} \theta^\top x
    \overset{(\ref{eq:thm2:whp})}{\geq} \theta_j^\top x .
\end{equation}
Combining (\ref{eq:th2:002}) and (\ref{eq:th2:003}), we obtain
\begin{equation}
    \label{eq:thm2:order}
    \tilde \theta_{t, j}^\top x_t \geq \theta_j^\top x .
\end{equation}
In the following, we consider two different scenarios, i.e., $ \theta_j^\top x \leq \theta_j^\top x_t $ and 
$ \theta_j^\top x > \theta_j^\top x_t $. For the former case, it is easy to show
 $$\mu_j(\theta_j^\top x) - \mu_j(\theta_j^\top x_t) \leq 0$$ since $\mu_j$ is monotonically increasing.
For the latter case, we have
\begin{equation*}
    \begin{split}
        & \mu_j(\theta_j^\top x) - \mu_j(\theta_j^\top x_t) \\
        \leq \; & L (\theta_j^\top x - \theta_j^\top x_t)
        \overset{(\ref{eq:thm2:order})}{\leq} L (\tilde \theta_{t, j}^\top x_t - \theta_j^\top x_t)  \\
        = \; & L  (\tilde \theta_{t, j} - \hat \theta_{t, j})^\top x_t + L
        (\hat \theta_{t, j} - \theta_j)^\top x_t  \\
        \leq \; & L({\Vert \tilde \theta_{t, j} - \hat \theta_{t, j} \Vert}_{Z_t} + {\Vert \hat \theta_{t, j} - \theta_j \Vert}_{Z_t})
         {\Vert x_t \Vert}_{Z_t^{-1}} \\
        \overset{(\ref{eq:thm1})}{\leq} & 2L\sqrt{\gamma_{t}} {\Vert x_t \Vert}_{Z_t^{-1}} 
        \leq 2L\sqrt{\gamma_{T+1}} {\Vert x_t \Vert}_{Z_t^{-1}}
    \end{split}
\end{equation*}
where the first inequality is due to the Lipschitz continuity of $\mu_j$, the third inequality follows from 
the H\"{o}lder's inequality, and the last inequality holds since $\gamma_{t}$ is monotonically increasing with $t$.
In summary, we have
\begin{equation*}
    \mu_j(\theta_j^\top x) - \mu_j(\theta_j^\top x_t) \leq 2 L\sqrt{\gamma_{T+1}} {\Vert x_t \Vert}_{Z_t^{-1}} .
\end{equation*}
Since the above inequality holds for any $x \in \mathcal{X}$, we have
$\Delta x_t \leq 2L\sqrt{\gamma_{T+1}} {\Vert x_t \Vert}_{Z_t^{-1}}$, which immediately implies
\begin{equation}
        \label{eq:thm2:delta}
        PR(T) = \sum_{t=1}^T \Delta x_t \leq 2L \sqrt{\gamma_{T+1}} \sum_{t=1}^T  {\Vert x_t \Vert}_{Z_t^{-1}} .
\end{equation}
We bound the RHS by the Cauchy--Schwarz inequality:
\begin{equation}
    \label{eq:thm2:cs}
    \sum_{t=1}^T  {\Vert x_t \Vert}_{Z_t^{-1}} \leq \sqrt{T \sum_{t=1}^T  {\Vert x_t \Vert}_{Z_t^{-1}}^2} .
\end{equation}
By Lemma 11 in \citeauthor{abbasi2011improved} \shortcite{abbasi2011improved}, we have
\begin{equation}
    \label{eq:lm11}
    \sum_{t=1}^T  {\Vert x_t \Vert}_{Z_t^{-1}}^2 \leq \frac{4}{\kappa} \log{\frac{\det{(Z_{T+1})}}{\det{(Z_1)}}} .
\end{equation}
Combining (\ref{eq:thm2:delta})-(\ref{eq:lm11}) and Corollary \ref{cor:1} finishes the proof. \hfill $\square$\par

\begin{figure*}[ht]

    \begin{subfigure}{0.32\textwidth}
          \centering
          \includegraphics[height=4.3cm,width=5.9125cm]{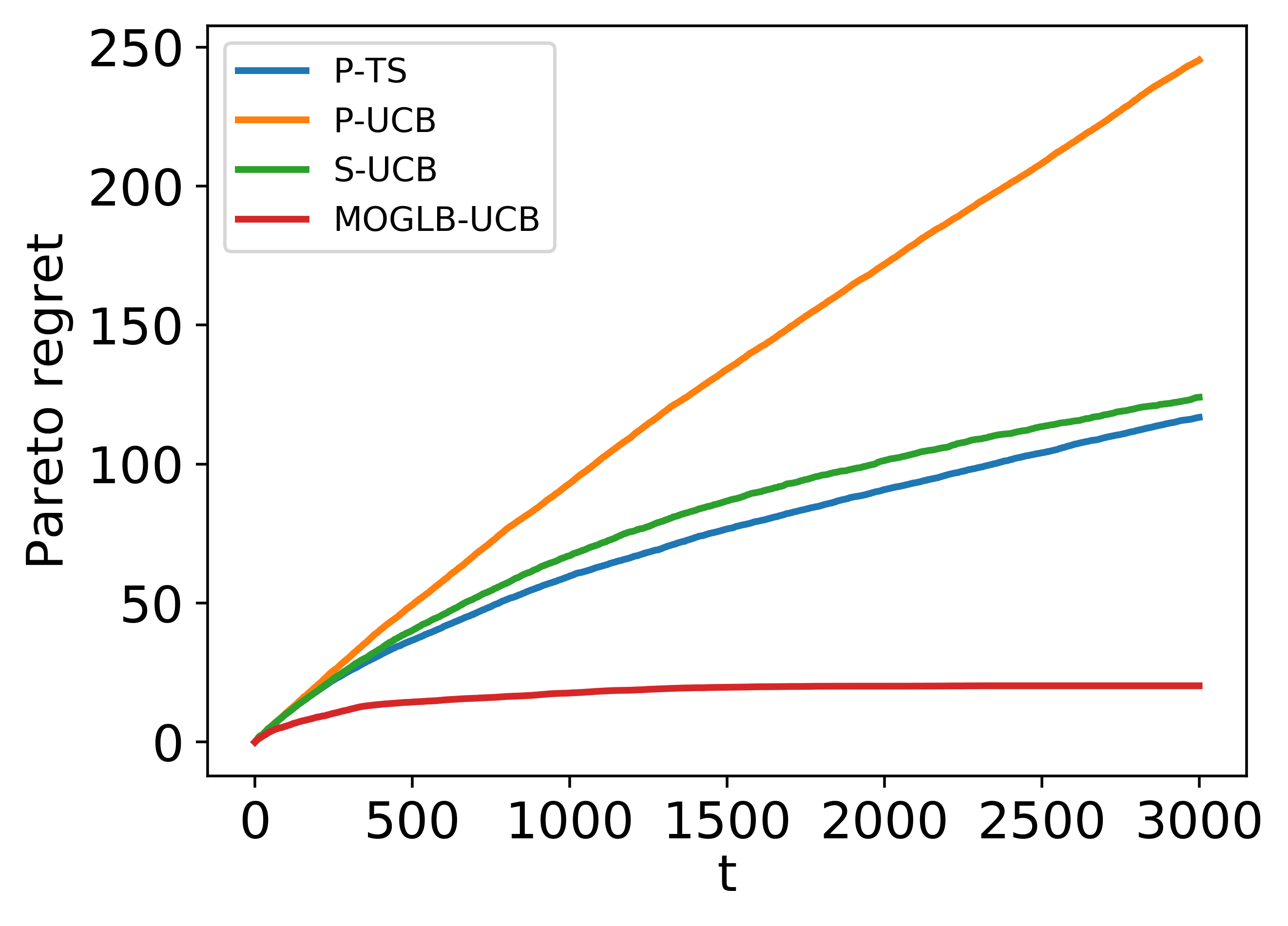}
        \caption{$d=5$}
    \end{subfigure}
    \hfill
    \begin{subfigure}{0.32\textwidth}
        \centering  
        \includegraphics[height=4.3cm,width=5.9125cm]{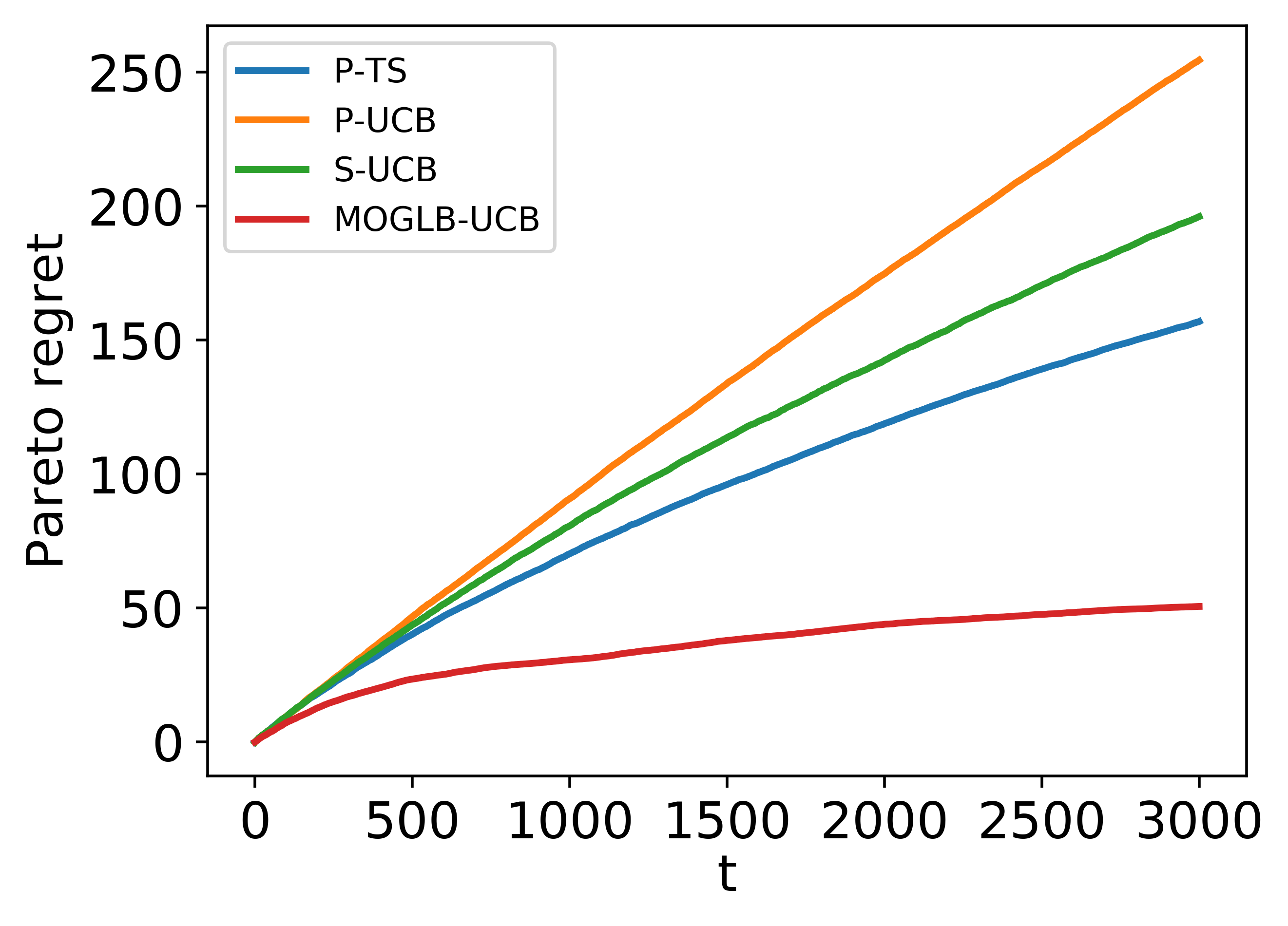}
        \caption{$d=10$}
    \end{subfigure}
    \hfill
    \begin{subfigure}{0.32\textwidth}
        \centering
        \includegraphics[height=4.3cm,width=5.9125cm]{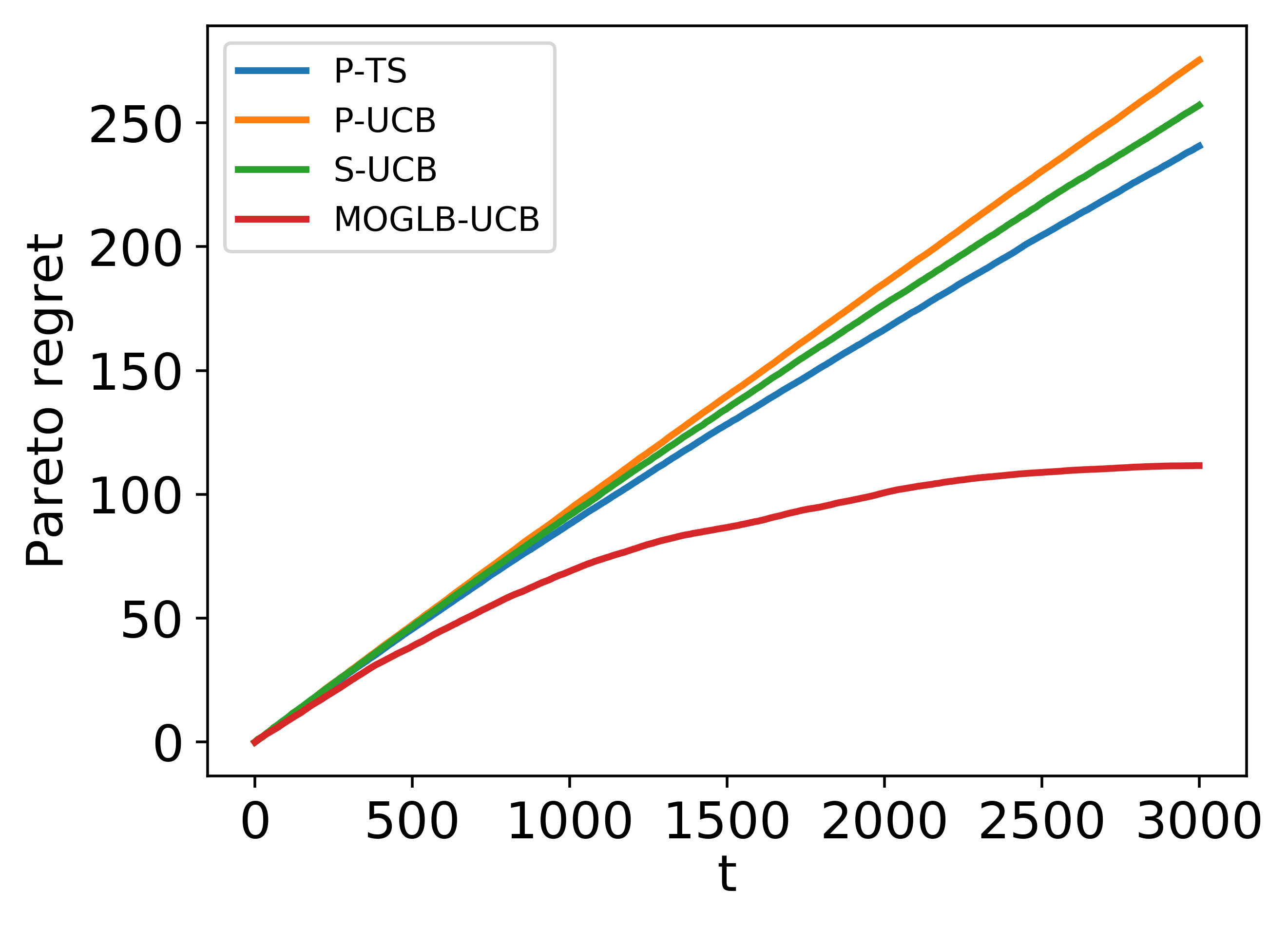}
      \caption{$d=15$}
    \end{subfigure}
    \caption{Pareto regret of different methods}
    \label{fig:exp1-pr}
\end{figure*}

\section{Experiments}
In this section, we conduct numerical experiments to compare our algorithm with the following multi-objective bandits algorithms.
\begin{itemize}
    \item P-UCB \cite{drugan2013designing}: This is the Pareto UCB algorithm, which compares different arms by the
    upper confidence bounds of their expected reward vectors and pulls an arm uniformly from the approximate Pareto front.
    \item S-UCB \cite{drugan2013designing}: This is the scalarized UCB algorithm, which scalarizes the reward vector
    by assigning weights to each objective and then employs the single objective UCB algorithm \cite{auer2002finite}. Throughout the experiments, we assign each objective with equal weight.
    \item P-TS \cite{yahyaa2015thompson}: This is the Pareto Thompson sampling algorithm, which makes use of the Thompson
     sampling technique to estimate the expected reward for every arm and selects an arm uniformly at random
    from the estimated Pareto front.
\end{itemize}
Note that the Pareto contextual zooming algorithm proposed by \citeauthor{turgay2018multi} \shortcite{turgay2018multi} is not included in the experiments, because one step of this algorithm is finding relevant balls whose specific implementation is lacked in their paper and no experimental results of their algorithm are reported as well.

In our algorithm, there is a parameter $\lambda$.
Since its functionality is just to make $Z_t$ invertible and our algorithm is insensitive to it, we simply
set $\lambda = \max(1, \kappa / 2)$.
Following common practice in bandits learning \cite{zhang2016online,jun2017scalable}, we also tune the width of the confidence set  $\gamma_t$ as $c\log{\frac{\det{(Z_{t})}}{\det{(Z_1)}}}$, where $c$ is searched within $[1e-3, 1]$.
We use a synthetic dataset constructed as follows.
Let $m=5$ and pick $d$ from $\{5, 10, 15\}$. 
For each objective $i \in [m]$, we sample the coefficients $\theta_i$ uniformly from the positive part of the unit ball.
To control the size of the Pareto front, we generate the arm set comprised of $4d$ arms as follows.
We first draw $3d$ arms uniformly from the centered ball whose radius is $0.5$, and then 
sample $d$ arms uniformly from the centered unit ball. We repeat this process until the size of the Pareto front is not more than $d$.

In each round $t=1,2,\ldots,T$, after the learner submits an arm $x_t$, he observes an $m$-dimensional reward vector, each 
component of which is generated according to the generalized linear model. 
While the GLM family contains various statistical models, in the experiments we choose two frequently used models namely 
the probit model and the logit model.
Specifically, the first two components of the reward vector are generated by the probit model, and
 the last three components are generated by the logit model. 

Since both the problem and the algorithms involve randomness, we perform $10$ trials up to round $T = 3000$ and report average performance
of the algorithms. As can be seen from Fig.~\ref{fig:exp1-pr}, where the vertical axis represents
the cumulative Pareto regret up to round $t$, our algorithm significantly outperforms its competitors
in all experiments. This is expected since all these algorithms are designed for multi-armed bandits problem and hence 
do not utilize the particular structure of the problem considered in this paper, which is explicitly exploited by our algorithm.

\begin{figure}[ht]
   \centering
   \includegraphics[height=6.2cm,width=8.525cm]{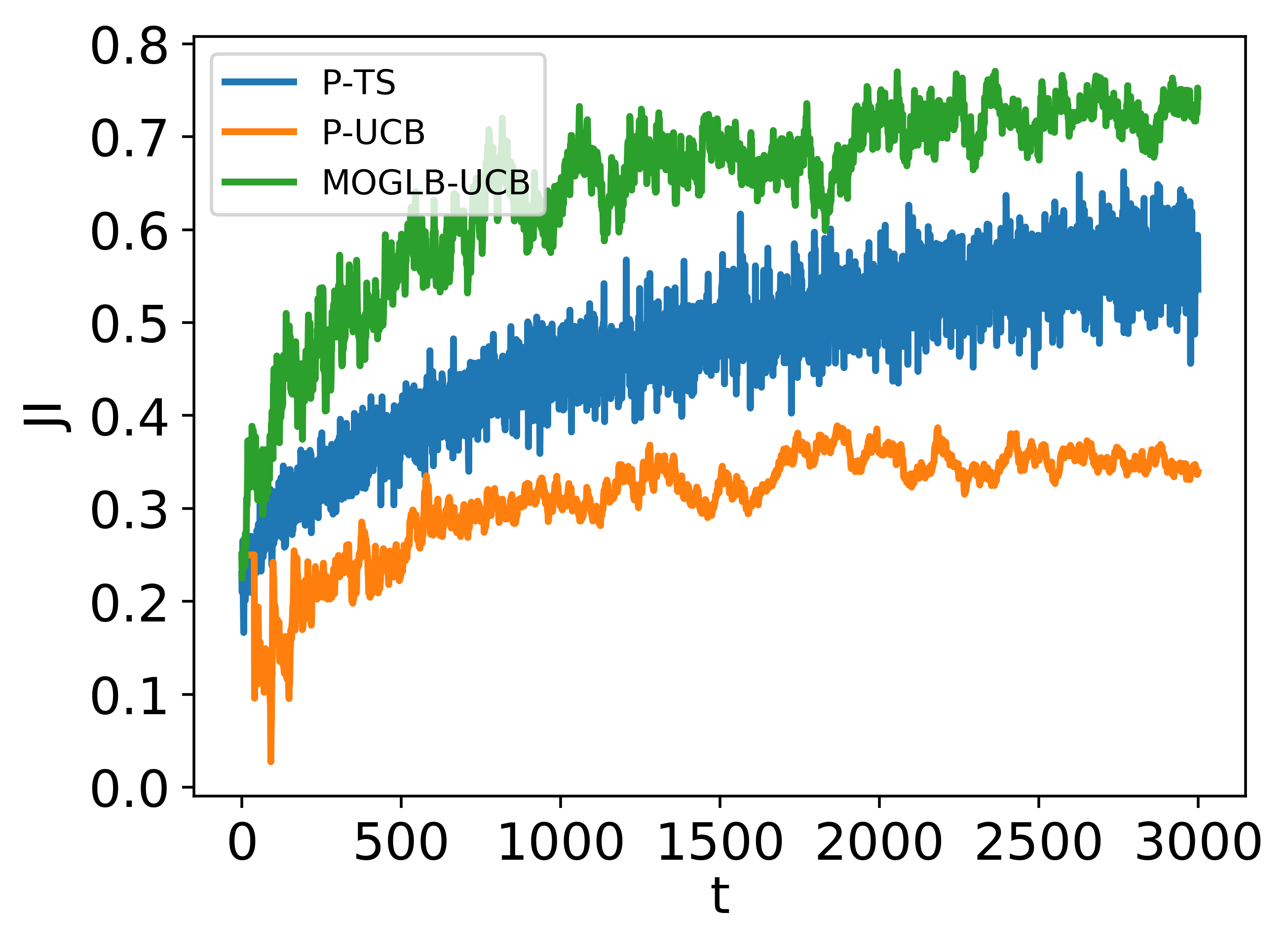}
   \caption{Jaccard index of different methods}
   \label{fig:exp1-ji}
\end{figure}

Finally, we would like to investigate the issue of fairness. To this end, we examine the approximate Pareto front $\mathcal{O}_{t}$ constructed by the tested algorithms except S-UCB and use  
Jaccard index (JI) to measure the similarity between $\mathcal{O}_{t}$ and the true Pareto front $\mathcal{O}^*$,
 defined as
 \begin{equation*}
    {\rm JI}_t \coloneqq \frac{\vert \mathcal{O}_{t} \cap \mathcal{O}^* \vert}{\vert \mathcal{O}_{t} \cup \mathcal{O}^* \vert}
\end{equation*}
for which the larger the better.

 We plot the curve of ${\rm JI}_t$ for each algorithm in Fig.~\ref{fig:exp1-ji}, where we set $d=10$.
 As can be seen, our algorithm finds the true Pareto front much faster than P-UCB and P-TS. Furthermore, the approximate Pareto front constructed by our algorithm is very close to the true Pareto front when $t > 1500$. Combining with the result shown in Fig.~\ref{fig:exp1-pr} and the uniform sampling strategy used in Step 3 of Algorithm \ref{alg:001}, we observe that our algorithm indeed minimizes the Pareto regret while ensuring fairness.

\section{Conclusion and Future Work}
In this paper, we propose a novel bandits framework named multi-objective generalized linear bandits, which extends 
the multi-objective bandits problem to contextual setting under the parameterized realizability assumption.
By employing the principle of ``optimism in face of uncertainty'', we develop a UCB-type algorithm whose Pareto regret is upper bounded by $\tilde O(d\sqrt{T})$, which matches the optimal regret bound for single objective contextual bandits problem. 

While we have conducted numerical experiments to show that the proposed algorithm is able to achieve high fairness, it is appealing to provide a theoretical guarantee regarding fairness. We will investigate this in future work.

\bibliographystyle{named}
\bibliography{MOGLB_Arxiv}

\begin{thebibliography}{}

\bibitem[\protect\citeauthoryear{Abbasi-Yadkori \bgroup \em et al.\egroup
  }{2011}]{abbasi2011improved}
Yasin Abbasi-Yadkori, D{\'a}vid P{\'a}l, and Csaba Szepesv{\'a}ri.
\newblock Improved algorithms for linear stochastic bandits.
\newblock In {\em Advances in Neural Information Processing Systems 24}, pages
  2312--2320, 2011.

\bibitem[\protect\citeauthoryear{Abbasi-Yadkori \bgroup \em et al.\egroup
  }{2012}]{abbasi2012online}
Yasin Abbasi-Yadkori, David Pal, and Csaba Szepesvari.
\newblock Online-to-confidence-set conversions and application to sparse
  stochastic bandits.
\newblock In {\em International Conference on Artificial Intelligence and
  Statistics}, pages 1--9, 2012.

\bibitem[\protect\citeauthoryear{Agarwal \bgroup \em et al.\egroup
  }{2014}]{agarwal2014taming}
Alekh Agarwal, Daniel Hsu, Satyen Kale, John Langford, Lihong Li, and Robert
  Schapire.
\newblock Taming the monster: A fast and simple algorithm for contextual
  bandits.
\newblock In {\em International Conference on Machine Learning}, pages
  1638--1646, 2014.

\bibitem[\protect\citeauthoryear{Auer \bgroup \em et al.\egroup
  }{2002a}]{auer2002finite}
Peter Auer, Nicolo Cesa-Bianchi, and Paul Fischer.
\newblock Finite-time analysis of the multiarmed bandit problem.
\newblock {\em Machine learning}, 47(2-3):235--256, 2002.

\bibitem[\protect\citeauthoryear{Auer \bgroup \em et al.\egroup
  }{2002b}]{auer2002nonstochastic}
Peter Auer, Nicolo Cesa-Bianchi, Yoav Freund, and Robert~E Schapire.
\newblock The nonstochastic multiarmed bandit problem.
\newblock {\em SIAM journal on computing}, 32(1):48--77, 2002.

\bibitem[\protect\citeauthoryear{Auer \bgroup \em et al.\egroup
  }{2016}]{auer2016pareto}
Peter Auer, Chao-Kai Chiang, Ronald Ortner, and Madalina Drugan.
\newblock Pareto front identification from stochastic bandit feedback.
\newblock In {\em International Conference on Artificial Intelligence and
  Statistics}, pages 939--947, 2016.

\bibitem[\protect\citeauthoryear{Auer}{2002}]{auer2002using}
Peter Auer.
\newblock Using confidence bounds for exploitation-exploration trade-offs.
\newblock {\em Journal of Machine Learning Research}, 3(Nov):397--422, 2002.

\bibitem[\protect\citeauthoryear{Bubeck and
  Cesa-Bianchi}{2012}]{bubeck2012regret}
S{\'e}bastien Bubeck and Nicolo Cesa-Bianchi.
\newblock Regret analysis of stochastic and nonstochastic multi-armed bandit
  problems.
\newblock {\em Foundations and Trends in Machine Learning}, 5(1):1--122, 2012.

\bibitem[\protect\citeauthoryear{Bubeck \bgroup \em et al.\egroup
  }{2009}]{bubeck2009online}
S{\'e}bastien Bubeck, Gilles Stoltz, Csaba Szepesv{\'a}ri, and R{\'e}mi Munos.
\newblock Online optimization in $\mathcal{X}$-armed bandits.
\newblock In {\em Advances in Neural Information Processing Systems 22}, pages
  201--208, 2009.

\bibitem[\protect\citeauthoryear{Dani \bgroup \em et al.\egroup
  }{2008}]{Dani2008StochasticLO}
Varsha Dani, Thomas~P. Hayes, and Sham~M. Kakade.
\newblock Stochastic linear optimization under bandit feedback.
\newblock In {\em Conference on Learning Theory}, page 355, 2008.

\bibitem[\protect\citeauthoryear{Drugan and Nowe}{2013}]{drugan2013designing}
MM~Drugan and A~Nowe.
\newblock Designing multi-objective multi-armed bandits algorithms: a study.
\newblock In {\em International Joint Conference on Neural Networks}, pages
  2352--2359, 2013.

\bibitem[\protect\citeauthoryear{Drugan and
  Now{\'e}}{2014}]{drugan2014scalarization}
Madalina~M Drugan and Ann Now{\'e}.
\newblock Scalarization based pareto optimal set of arms identification
  algorithms.
\newblock In {\em International Joint Conference on Neural Networks}, pages
  2690--2697, 2014.

\bibitem[\protect\citeauthoryear{Dudik \bgroup \em et al.\egroup
  }{2011}]{Dudik11}
Miroslav Dudik, Daniel Hsu, Satyen Kale, Nikos Karampatziakis, John Langford,
  Lev Reyzin, and Tong Zhang.
\newblock Efficient optimal learning for contextual bandits.
\newblock In {\em Conference on Uncertainty in Artificial Intelligence}, pages
  169--178, 2011.

\bibitem[\protect\citeauthoryear{Filippi \bgroup \em et al.\egroup
  }{2010}]{filippi2010parametric}
Sarah Filippi, Olivier Cappe, Aur{\'e}lien Garivier, and Csaba Szepesv{\'a}ri.
\newblock Parametric bandits: The generalized linear case.
\newblock In {\em Advances in Neural Information Processing Systems 23}, pages
  586--594, 2010.

\bibitem[\protect\citeauthoryear{Hazan \bgroup \em et al.\egroup
  }{2007}]{hazan2007logarithmic}
Elad Hazan, Amit Agarwal, and Satyen Kale.
\newblock Logarithmic regret algorithms for online convex optimization.
\newblock {\em Machine Learning}, 69(2-3):169--192, 2007.

\bibitem[\protect\citeauthoryear{Jun \bgroup \em et al.\egroup
  }{2017}]{jun2017scalable}
Kwang-Sung Jun, Aniruddha Bhargava, Robert Nowak, and Rebecca Willett.
\newblock Scalable generalized linear bandits: Online computation and hashing.
\newblock In {\em Advances in Neural Information Processing Systems 30}, pages
  99--109, 2017.

\bibitem[\protect\citeauthoryear{Kleinberg \bgroup \em et al.\egroup
  }{2008}]{kleinberg2008multi}
Robert Kleinberg, Aleksandrs Slivkins, and Eli Upfal.
\newblock Multi-armed bandits in metric spaces.
\newblock In {\em Proceedings of the 40th Annual ACM Symposium on Theory of
  Computing}, pages 681--690, 2008.

\bibitem[\protect\citeauthoryear{Langford and Zhang}{2008}]{langford2008epoch}
John Langford and Tong Zhang.
\newblock The epoch-greedy algorithm for multi-armed bandits with side
  information.
\newblock In {\em Advances in Neural Information Processing Systems 21}, pages
  817--824, 2008.

\bibitem[\protect\citeauthoryear{Li \bgroup \em et al.\egroup
  }{2010}]{li2010contextual}
Lihong Li, Wei Chu, John Langford, and Robert~E Schapire.
\newblock A contextual-bandit approach to personalized news article
  recommendation.
\newblock In {\em International Conference on World Wide Web}, pages 661--670,
  2010.

\bibitem[\protect\citeauthoryear{Lu \bgroup \em et al.\egroup
  }{2010}]{lu2010contextual}
Tyler Lu, D{\'a}vid P{\'a}l, and Martin P{\'a}l.
\newblock Contextual multi-armed bandits.
\newblock In {\em International Conference on Artificial Intelligence and
  Statistics}, pages 485--492, 2010.

\bibitem[\protect\citeauthoryear{Lu \bgroup \em et al.\egroup
  }{2019}]{pmlr-v97-lu19c}
Shiyin Lu, Guanghui Wang, Yao Hu, and Lijun Zhang.
\newblock Optimal algorithms for {L}ipschitz bandits with heavy-tailed rewards.
\newblock In {\em Proceedings of the 36th International Conference on Machine
  Learning}, pages 4154--4163, 2019.

\bibitem[\protect\citeauthoryear{Nelder and Wedderburn}{1972}]{Neld1972GLM}
J.~A. Nelder and R.~W.~M. Wedderburn.
\newblock Generalized linear models.
\newblock {\em Journal of the Royal Statistical Society, Series A, General},
  135:370--384, 1972.

\bibitem[\protect\citeauthoryear{Rodriguez \bgroup \em et al.\egroup
  }{2012}]{rodriguez2012multiple}
Mario Rodriguez, Christian Posse, and Ethan Zhang.
\newblock Multiple objective optimization in recommender systems.
\newblock In {\em Proceedings of the sixth ACM conference on Recommender
  systems}, pages 11--18. ACM, 2012.

\bibitem[\protect\citeauthoryear{Slivkins}{2014}]{slivkins2014contextual}
Aleksandrs Slivkins.
\newblock Contextual bandits with similarity information.
\newblock {\em Journal of Machine Learning Research}, 15(1):2533--2568, 2014.

\bibitem[\protect\citeauthoryear{Turgay \bgroup \em et al.\egroup
  }{2018}]{turgay2018multi}
Eralp Turgay, Doruk Oner, and Cem Tekin.
\newblock Multi-objective contextual bandit problem with similarity
  information.
\newblock In {\em International Conference on Artificial Intelligence and
  Statistics}, pages 1673--1681, 2018.

\bibitem[\protect\citeauthoryear{Yahyaa and
  Manderick}{2015}]{yahyaa2015thompson}
Saba Yahyaa and Bernard Manderick.
\newblock Thompson sampling for multi-objective multi-armed bandits problem.
\newblock In {\em European Symposium on Artificial Neural Networks}, pages
  47--52, 2015.

\bibitem[\protect\citeauthoryear{Zhang \bgroup \em et al.\egroup
  }{2016}]{zhang2016online}
Lijun Zhang, Tianbao Yang, Rong Jin, Yichi Xiao, and Zhi-hua Zhou.
\newblock Online stochastic linear optimization under one-bit feedback.
\newblock In {\em International Conference on Machine Learning}, pages
  392--401, 2016.

\end{thebibliography}

\end{document}